\documentclass{article}
\usepackage{fullpage}
\usepackage{algorithm,algorithmic}
\usepackage{amssymb,amsmath,amsthm}
\usepackage{xcolor}

\newtheorem{theorem}{Theorem}
\newtheorem{assumption}[theorem]{Assumption} 
\newtheorem{lemma}[theorem]{Lemma}

\newtheorem{definition}[theorem]{Definition}

\newcommand{\poly}{\textrm{poly}}
\newcommand{\argmin}{\textrm{argmin}}

\newcommand{\reals}{\mathbb{R}}
\newcommand{\regret}{\textrm{Regret}}

\newcommand{\mI}{\mathcal I}
\newcommand{\citep}{\cite}

\begin{document}

\title{The Nonstochastic Control Problem}

\author{
    Elad Hazan$^{1,2}$ \qquad Sham M. Kakade$^{1,3,4}$ \qquad Karan Singh$^{1,2}$\\
    \quad \\
    $^1$ Google AI Princeton \\
    $^2$ Department of Computer Science, Princeton University \\
    $^3$ Allen School of Computer Science and Engineering, University of Washington\\
    $^4$ Department of Statistics, University of Washington\\ \texttt{\{ehazan, karans\}@princeton.edu, sham@cs.washington.edu}\\
}
\maketitle

\newcommand{\elad}[1]{$\ll$\textsf{\color{red} Elad : #1}$\gg$}

\begin{abstract}%
We consider the problem of controlling an unknown linear dynamical system in the presence of (nonstochastic) adversarial perturbations and adversarial convex loss functions. In contrast to classical control, here it is impossible to precompute the optimal controller as it depends on the yet unknown perturbations and costs. Instead, we measure regret against an optimal linear policy in hindsight, and give the first efficient algorithm that guarantees a sublinear regret bound, scaling as $O(T^{2/3})$, in this setting.
\end{abstract}

\section{Introduction}
 Classical control theory assumes that nature evolves according to well-specified dynamics that is perturbed by i.i.d. noise. While this approximation has proven very useful for controlling some real world systems, it does not allow for construction of truly robust controllers. The focus of this paper is the construction of truly robust controllers {\bf even when the underlying system is unknown and the perturbations are adversarially chosen}.  For this purpose we describe the nonstochastic control problem and study efficient algorithms to solve it for linear dynamical systems. 

Specifically, we consider the case in which the underlying system is linear, but has potentially adversarial perturbations (that can model deviations from linearity), i.e.
\begin{equation} \label{eqn:shalom}
x_{t+1} = A x_t + B u_t + w_t ,
\end{equation}
where $x_t$ is the (observed) dynamical state, $u_t$ is a learner-chosen control and $w_t$ is an adversarial disturbance. The goal of the controller is to minimize a sum of sequentially revealed adversarial cost functions $c_t(x_t,u_t)$ over the state-control pairs that it visits.

The adversarial nature of $w_t$ prohibits an a priori computation of an optimal policy that is the hallmark of classical optimal control. Instead, we consider algorithms for online control that iteratively produce a control $u_t$ based on previous observations. The goal in this game-theoretic setting is to minimize policy regret, or the regret compared to the best controller from a class $\Pi$, chosen with complete foreknowledge of the system dynamics, the cost sequence, and all the disturbances: 
$$ \regret =  \sum_{t=1}^T c_t (x_t ,u_t) - \min_{\pi \in \Pi} \sum_{t=1}^T c_t(x^\pi_t , u^\pi_t) . $$

Notice that the cost of the benchmark is measured on the counterfactual state-action sequence $(x^\pi_t, u^\pi_t)$ that the benchmark policy in consideration visits, as opposed to the state-sequence visited by the the learner. This implies instance-wise near-optimality for every perturbation sequence. 

We give a formal definition of the nonstochastic control problem henceforth. Informally, nonstochastic control for linear dynamical systems (LDS) can be stated as follows: 

\medskip
\noindent\fbox{%
    \parbox{\textwidth}{%
        {\bf Nonstochastic Control for LDS:} Without knowledge of the underlying system $A,B$, or the perturbations $w_t$, iteratively generate controls $u_t$ to minimize regret, over sequentially revealed adversarial convex costs $c_t(x_t,u_t)$, against the class of all linear policies.
    }%
}
\smallskip

Our main result is an efficient algorithm for the nonstochastic control problem which attains the following guarantee:
\begin{theorem}[Informal Statement]
For an unknown linear dynamical system where the perturbations $w_t$ (and convex costs $c_t$) are bounded and chosen by an  adversary, there exists an efficient algorithm that generates an adaptive sequence of controls $\{u_t\}$ for which 
$$ \regret  = O(\poly(\texttt{natural-parameters})T^{2/3}) .$$
\end{theorem}

\subsection{Technical Outline}

Our starting point is the recent work of \cite{agarwal2019online}, where the authors proposed a novel class of policies choosing actions as a linear combination of past perturbations, $u_t = \sum_{i=1}^k M_i w_{t-i}$. They demonstrate that learning the coefficients $M_i$, via online convex optimization, allows their controller to compete with the class of all linear state-feedback policies. This latter class is important, since it is known to be optimal for the standard setting of normal i.i.d noise and quadratic loss functions, also known as the Linear Quadratic Regulator (LQR), and associated robust control settings (see \cite{bacsar2008h} for examples).

The caveat in \cite{agarwal2019online} is that the system matrices $(A,B)$ need to be known. In the case of a known system, the disturbances can be simply computed via observations of the state, ie. $w_t = x_{t+1} - A x_t - B u_t$. However, if the system is unknown, it is not clear how to generalize their approach. Fundamentally, the important component that is difficulty in identifying the system, or the matrices $A,B$, from the observations. This is non-trivial since the noise is assumed to be adversarial, and was posed as a question in \cite{tu2019sample}. 

In this paper we show how to overcome this difficulty and obtain sublinear regret for controlling an unknown system in the presence of adversarial noise and adversarial loss functions. The regret notion we adopt is policy regret against linear policies, exactly as in \cite{agarwal2019online}. An important component that we use is {\it adversarial sys-id}: an efficient method for uncovering the underlying system even in the presence of adversarial perturbations. This method is {\bf not} based on naive least squares method of regressing $(x_t,u_t)$ on $x_{t+1}$. In particular, without independent, zero-mean $w_t$'s, the latter approach can produce inconsistent estimates of the system matrices.


\subsection{Related Work}
\paragraph{Robust Control:} The classical control literature deals with  adversarial perturbations in the dynamics in a framework known as  $H_\infty$ control, see e.g. \cite{z2,z1}. In this setting, the controller solves for the best linear controller assuming worst case noise to come, i.e.,
$$ \min_{K_1} \max_{w_{1} } \min_{K_2} ... \min_{K_{T}} \max_{w_{T} } \sum_t c_t(x_t,u_t) . $$
This approach is overly pessimistic as it optimizes over the worst-case noise. In contrast, the metric in nonstochastic control (NSC) is regret, which adapts to the per-instance perturbations.  

\paragraph{Learning to control stochastic LDS:} 
There has been a resurgence of literature on control of linear dynamical systems in the recent machine learning venues. The case of known systems was extensively studied in the control literature, see the survey \cite{z2}. Sample complexity and regret bounds for control (under Gaussian noise) were obtained in \cite{abbasi2011regret,dean2018regret,a2,mania2019certainty,cohen2019learning}. 
The works of \cite{abbasi2014tracking}, \cite{cohen2018online} and \cite{agarwal2019logarithmic} allow for control in LDS with adversarial loss functions. Provable control in the Gaussian noise setting via the policy gradient method was  studied in \cite{fazel2018global}. These works operate in the absence of perturbations or assume that the same are i.i.d., as opposed to our adversarial. 

\paragraph{Control with adversarial perturbations:}
The most relevant reformulation of the control problem that enables our result is the recent work of \cite{agarwal2019online}, who use online learning techniques and convex relaxation to obtain provable bounds for controlling LDS with adversarial perturbations. However, the result and the algorithm make extensive use of the availability of the system matrices. \cite{tu2019sample} ask if the latter result can be extended to unknown systems, a question that we answer in the affirmative.  

\paragraph{System identification.}
For the stochastic setting, several works \citep{faradonbeh2018finite, simchowitz2018learning, sarkar2019near} propose to use the least-squares procedure for parameter identification. In the adversarial setting, least-squares can lead to inconsistent estimates. For the partially observed stochastic setting, \cite{oymak2019non, sarkar2019finite, simchowitz2018learning} give results guaranteeing parameter recovery using Gaussian inputs. Of these, the results in \cite{simchowitz2019learning} also apply to the adversarial setting. We offer a simpler analysis for parameter identification, and develop rigorous perturbation bounds for the control algorithm necessary to make guarantees on the quality of the control solution. Other relevant work from the machine learning literature includes spectral filtering techniques for learning and open-loop control of partially observable systems \citep{hazan2017learning, arora2018towards, hazan2018spectral}. 

\section{Problem Definition}\label{s:setup}
\subsection{Nontochastic Control: The General Case} 
Many dynamical systems admit the following dynamical description. These may be seen as discrete-time analogues of controlled diffusion processes \citep{krylov2008controlled}.
$$ x_{t+1} =  f(x_t,u_t) + w_t$$ 
Here $f(x,u)$ is a transition function, and $w_t$ are perturbations or deviations from nominal dynamics. We consider an online control problem during the course of which a controller must iteratively choose a control input $u_t \in \reals^n$, and suffers a loss $c_t(x_t,u_t)$, where $x_t \in \reals^m$ is the state of the system. The controller is thereafter presented some output in the form of a resultant information set $\mI_t$. For example, the controller may observe the state $x_t$, and cost function $c_t$, but the transition function $f$ might be unknown. A policy $\pi=(\pi_1,\dots \pi_t : \pi_t:: \mI_t \to \reals^m)$ is a mapping from observed information to control. We denote a set of policies by $\Pi$. We measure the performance of a control algorithm through the metric of policy regret: the difference between the aggregate cost of the controller and that of the best policy in hindsight from a certain class.  
\begin{definition}[Nonstochastic Control]
A Nonstochastic Control problem instance is given by tuple $( f:\reals^m\times \reals^n \to \reals^m ,\{w_t\in \reals^m \}, \{c_t: \reals^m\times \reals^n \to \reals\}, \Pi ) $, where for any sequence of controls $u_1,...,u_T$, the states are produced as $x_{t+1} =  f(x_t,u_t) + w_t$. The goal of the learner is to choose an adaptive sequence of controls to minimize regret against the policy class $\Pi$, defined as: 
$$ \regret =  \sum_{t=1}^T c_t (x_t ,u_t) - \min_{\pi \in \Pi} \sum_{t=1}^T c_t(x^\pi_t , u^\pi_t),$$
where $(x_t^\pi, u_t^\pi)$ is the state-control pair visited by the benchmark policy $\pi\in \Pi$ in consideration.
\end{definition}

We specialize the above definition to linear dynamical systems below.

\subsection{Nonstochastic Control for Linear Dynamical Systems}
We consider the setting of linear dynamical systems with time-invariant dynamics, i.e. 
$$ x_{t+1} = A x_t + B u_t + w_t,$$
where $x_t\in \reals^m$ and $u_t\in \reals^n$. The perturbation sequence $w_t$ may be adversarially chosen at the beginning of the interaction, and is unknown to the learner. Likewise, the system is augmented with time-varying convex cost functions $c_t(x,u)$. The total cost associated with a sequence of (random) controls, derived through an algorithm $\mathcal{A}$, is 
\[  J(\mathcal{A}) = \sum_{t=1}^T c_t(x_t,u_t).\]
With some abuse of notation, we will denote by $J(K)$ the cost associated with the execution of controls as a linear controller $K$ would suggest, ie. $u_t = -Kx_t$. The following conditions are assumed on the cost and the perturbations~\footnote{Without loss of generality we shall assume that $G,D,W, \kappa\geq 1$ holds, since these are upper bounds.}.

\begin{assumption}\label{a:a1}
The perturbation sequence is bounded, ie. $\|w_t\|\leq W$, and chosen at the start of the interaction, implying that this sequence $w_t$ does not depend on the choice of $u_t$.
\end{assumption}

\begin{assumption}\label{a:a2}
As long as $\|x_t\|,\|u_t\|\leq D$, the convex costs admit  $\|\nabla_{(x,u)} c_t(x,u)\|\leq GD$.
\end{assumption}

The fundamental Linear Quadratic Regulator problem is a specialization of the above to the case when the perturbations are i.i.d. Gaussian and the cost functions are positive quadratics, ie.
$$ c_t(x,u) = x^\top Q x + u^\top R u.$$

\paragraph{Objective} We consider the setting where the learner has no knowledge of $A,B$ and the perturbation sequence $w_t$. In this case, any inference of these quantities may only take place indirectly through the observation of the state $x_t$. Furthermore, the learner is made aware of the cost function $c_t$ only once the choice of $u_t$ has been made.

Under such constraints, the objective of the algorithm is to choose an (adaptive) sequence of controls that ensure that the cost suffered in this manner is comparable to that of the best choice of a linear controller with complete knowledge of system dynamics $A,B$ and the foreknowledge of the cost and perturbation sequences $\{c_t,w_t\}$. Formally, we measure regret as 
$$  \mbox{Regret} = J(\mathcal{A}) - \min_{K\in \mathcal{K}} J(K). $$
$\mathcal{K}$ is the set of $(\kappa,\gamma)$-strongly stable linear controllers defined below. The notion of strong stability, introduced in \cite{cohen2018online}, offers a quantification of the classical notion of a stable controller in manner that permits a discussion on non-asymptotic regret bounds.

\begin{definition}[Strong Stability]
A linear controller $K$ is $(\kappa,\gamma)$-strongly stable for a linear dynamical system specified via $(A,B)$ if there exists a decomposition of $A-BK=QLQ^{-1}$ with $\|L\|\leq 1-\gamma$, and $\|A\|,\|B\|,\|K\|,\|Q\|,\|Q^{-1}\|\leq \kappa$. 
\end{definition}

We also assume the learner has access to a fixed stabilizing controller $\mathbb{K}$. When operating under unknown transition matrices, the knowledge of a stabilizing controller permits the learner to prevent an inflation of the size of the state beyond reasonable bounds.

\begin{assumption}\label{a:a4}
The learner knows a linear controller $\mathbb{K}$ that is $(\kappa,\gamma)$-strongly stable for the true, but unknown, transition matrices $(A,B)$ defining the dynamical system.
\end{assumption}

The non-triviality of the regret guarantee rests on the benchmark set not being empty. As noted in \cite{cohen2018online}, a sufficient condition to ensure the existence of a strongly stable controller is the controllability of the linear system $(A,B)$. Informally, controllability for a linear system is characterized by the ability to drive the system to any desired state through appropriate control inputs in the presence of deterministic dynamics, ie. $x_{t+1} = Ax_t+Bu_t$. 

\begin{definition}[Strong Controllability]
For a linear dynamical system $(A,B)$, define, for $k\geq 1$, a matrix $C_k\in \reals^{n\times km}$ as 
$$C_k=[B, AB, A^2B\dots A^{k-1}B].$$
A linear dynamical system $(A,B)$ is controllable with controllability index $k$ if $C_k$ has full row-rank. In addition, such a system is also $(k,\kappa)$-strongly controllable if $\|(C_k C_k^\top)^{-1}\|\leq \kappa$.
\end{definition}

As with stability, a quantitative analog of controllability first suggested in \cite{cohen2018online} is presented above. It is useful to note that, as a consequence of the Cayley-Hamiltion theorem, for a controllable system the controllability index is always at most the dimension of the state space. We adopt the assumption that the system $(A-B\mathbb{K},B)$ is $(k,\kappa)$ strongly controllable.

\begin{assumption}\label{a:a3}
The linear dynamical system $(A-B\mathbb{K},B)$ is $(k,\kappa)$-strongly controllable.
\end{assumption}

\section{Preliminaries}
This section sets up the concepts that aid the algorithmic description and the analysis. 

\subsection{Parameterization of the Controller}
The total cost objective of a linear controller is non-convex in the canonical parameterization \citep{fazel2018global}, ie. $J(K)$ is not convex in $K$. To remedy this, we use an alternative perturbation-based parameterization for controller, recently proposed in \cite{agarwal2019online}, where the advised control is linear in the past perturbations (as opposed to the state). This permits that the offline search for an optimal controller may be posed as a convex program.

\begin{definition}
A perturbation-based policy $M=(M^{[0]},\dots M^{[H-1]})$ chooses control $u_t$ at state $x_t$,
\[u_t = -\mathbb{K}x_t + \sum_{i=1}^{H}M^{[i-1]}w_{t-i}.\]
\end{definition}

\subsection{State Evolution}
Under the execution of a stationary policy $M$, the state may be expressed as a linear transformation $\Psi$ (defined below) of the perturbations, where the transformation $\Psi$'s is linear in the  matrices $M$'s.

\[x_{t+1} = (A-B\mathbb{K})^{H+1}x_{t-H} + \sum_{i=0}^{2H} \Psi_{i}(M|A,B)w_{t-i}\]

\begin{definition}
For a matrix pair $(A,B)$, define the state-perturbation transfer matrix:
\[\Psi_{i}(M|A,B) = (A-B\mathbb{K})^i \mathbf{1}_{i\leq H}  + \sum_{j=0}^H (A-B\mathbb{K})^j B M^{[i-j-1]} \mathbf{1}_{i-j\in [1,H]}.\]
\end{definition}

\begin{definition}
Define the surrogate state $y_{t+1}$ and the surrogate action $v_{t}$ as stated below. The surrogate cost $f_t$ as chosen to be the specialization of the $t$-th cost function with the surrogate state-action pair as the argument.
\begin{align*}
y_{t+1}(M|A,B,\{w\}) &= \sum_{i=0}^{2H} \Psi_{i}(M|A,B)w_{t-i}\\
v_{t}(M|A,B,\{w\}) &= -\mathbb{K}y_t(M|A,B,\{w\}) + \sum_{i=1}^{H} M^{[i-1]}w_{t-i}\\
f_t(M|A,B,\{w\}) &= c_t(y_t(M|A,B,\{w\}), v_t(M|A,B,\{w\}))
\end{align*}
\end{definition}

\section{The Algorithm}

Our approach follows the explore-then-commit paradigm, identifying the underlying the deterministic-equivalent dynamics to within some accuracy using random inputs in the exploration phase. Such an approximate recovery of parameters permits an approximate recovery of the perturbations, thus facilitating the execution of the perturbation-based controller on the approximated perturbations. 

\begin{algorithm}[ht!]
	\caption{Adversarial control via system identification.}
	\label{alg:complete}
	\begin{algorithmic}[]
		\STATE \textbf{Input:} learning rate $\eta$, horizon $H$, number of iterations $T$, rounds of exploration $T_0$.
		\STATE \textbf{\underline{Phase 1:} System Identification.}
		\STATE Call Algorithm~\ref{alg:B} with a budget of $T_0$ rounds to obtain system estimates $\hat{A},\hat{B}$.
		\STATE \textbf{\underline{Phase 2:} Robust Control.}
		\STATE Define the constraint set $\mathcal{M}$ as $\mathcal{M} = \{M=(M^{[0]},\dots M^{[H-1]}): \|M^{[i-1]}\|\leq \kappa^4(1-\gamma)^i\}$.
		\STATE Initialize $\hat{w}_{T_0}=x_{T_0+1}$ and $\hat{w}_t=0$ for $t< T_0$.
		\FOR{$t = T_0+1, \ldots, T$}
		\STATE Choose the action
		\[u_t = -\mathbb{K}x_t + \sum_{i=1}^H M_t^{[i-1]}\hat{w}_{t-i}.\]
		\STATE Observe the new state $x_t$, the cost function $c_t(x,u)$.
		\STATE Record an estimate $\hat{w}_{t} = x_{t+1}  -\hat{A} x_t -\hat{B}u_t$.
		\STATE Update $M_{t+1}=\Pi_\mathcal{M} (M_t - \eta\nabla f_t(M_t | \hat{A},\hat{B},\{\hat{w}\}))$.
		\ENDFOR
	\end{algorithmic}
\end{algorithm}

\begin{algorithm}
	\caption{System identification via random inputs.}
	\label{alg:B}
	\begin{algorithmic}[]
		\STATE Input: number of iterations $T_0$.
		\FOR{$t = 0, \ldots, T_0$}
		\STATE Execute the control $u_t  = -\mathbb{K}x_t+\eta_t$ with $\eta_t \sim_{i.i.d.}
		\{\pm 1\}^{n}$.
		\STATE Record the observed state $x_t$. 
		\ENDFOR
		\STATE Declare $N_j = \frac{1}{T_0-k}\sum_{t=0}^{T_0-k-1} x_{t+j+1} \eta_{t}^\top $, for all $j \in [k] $.
		\STATE Define $C_0=(N_0, \dots N_{k-1})$, $C_1=(N_1,\dots N_k)$, and return $\hat{A},\hat{B}$ as
		\[ \hat{B} = N_0,\quad \hat{A'} = C_1 C_0^\top (C_0 C_0^\top)^{-1},	\quad \hat{A} = \hat{A'}+\hat{B}\mathbb{K}. \]
	\end{algorithmic}
\end{algorithm}

\begin{theorem} \label{thm:main}
Under the assumptions \ref{a:a1}, \ref{a:a2}, \ref{a:a3}, \ref{a:a4}, when $H=\Theta(\gamma^{-1}\log (\kappa^2 T))$, $\eta = \Theta(GW\sqrt{T})^{-1}$, $T_0=\Theta(T^{2/3}\log{\delta^{-1}})$, the regret incurred by Algorithm~\ref{alg:complete} for controlling an unknown linear dynamical system admits the upper bound~\footnote{If one desires the regret to scale as $GW^2$, as may be rightly demanded due to Assumption~\ref{a:a2}, it suffices to choose the exploration scheme in Algorithm~\ref{alg:B} as $\{\pm W\}^n$, while introducing a multiplicative factor of $W^{-2}$ in the computation $N_j$'s. Such modifications ensure a natural scaling for all terms involving costs.} stated below with probability at least $1-\delta$.
$$ \textrm{Regret} = O(\poly(\kappa,\gamma^{-1},k,m,n,G,W)T^{2/3}\log \delta^{-1})$$
\end{theorem}

\section{Regret Analysis}
To present the proof concisely, we set up a few articles of use. For a generic algorithm $\mathcal{A}$ operating on a generic linear dynamical system specified via a matrix pair $(A,B)$ and perturbations $\{w\}$, let
\begin{enumerate}
    \item $J(\mathcal{A}|A,B,\{w\})$ be the cost of executing $\mathcal{A}$, as incurred on the last $T-T_0$ time steps,
    \item $x_t(\mathcal{A}|A,B,\{w\})$ be the state achieved at time step $t$, and
    \item $u_t(\mathcal{A}|A,B,\{w\})$ be the control executed at time step $t$.
\end{enumerate}

We also note that following result from~\cite{agarwal2019online} that applies to the case when the matrices $(A,B)$ that govern the underlying dynamics are made known to the algorithm. 

\begin{theorem}[Known System; \cite{agarwal2019online}]\label{thm:old}
Let $\mathbb{K}$ be a $(\kappa,\gamma)$-strong stable controller for a system $(A,B)$, $\{w_t\}$ be a (possibly adaptive) perturbation sequence with $\|w_t\|\leq W$, and $c_t$ be costs satisfying Assumption~\ref{a:a2}. Then there exists an algorithm $\mathcal{A}$ (Algorithm \ref{alg:complete}.2 with a learning rate of $\eta=\Theta(GW\sqrt{T})^{-1}$ and $H=\Theta(\gamma^{-1}\log (\kappa^2 T))$), utilizing $\mathbb{K},(A,B)$, that guarantees
\[J(\mathcal{A}|A,B,\{w\}) - \min_{K\in \mathcal{K}} J(K|A,B,\{w\})\leq O(\poly(m,n,\kappa,\gamma^{-1})GW^2\sqrt{T}\log T).\]
\end{theorem}

\begin{proof}{of Theorem~\ref{thm:main}.}
Define $K= \argmin_{K\in\mathcal{K}} J(K)$. Let $J_0$ be the contribution to the {\em regret} associated with the first $T_0$ rounds of exploration. By Lemma~\ref{l:exp}, we have that $J_0 \leq 16 T_0 Gn\kappa^8\gamma^{-2}W^2$. 

Let $\mathcal{A}$ refer to the algorithm, from \cite{agarwal2019online}, executed in Phase 2. By Lemma \ref{l:sim},
\begin{align*}
    \textrm{Regret} &\leq J_0 + \sum_{t=T_0+1}^T c_t(x_t, u_t) - J(K|A,B,\{w\}), \\
    &\leq J_0 + (J(\mathcal{A}|\hat{A},\hat{B},\{\hat{w}\}) - J(K|\hat{A},\hat{B},\{\hat{w}\})) + (J(K|\hat{A},\hat{B},\{\hat{w}\}) - J(K|A,B,w)).
\end{align*}

Let $\|A-\hat{A}\|,\|B-\hat{B}\|\leq \varepsilon_{A,B}$ and $\varepsilon_{A,B}\leq 10^{-3}\kappa^{-10}\gamma^2$ in the arguments below. The middle term above can be upper bounded by the regret of algorithm $\mathcal{A}$ on the fictitious system $(\hat{A},\hat{B})$ and the perturbation sequence $\{\hat{w}\}$. Before we can invoke Theorem~\ref{thm:old}, observe that 
\begin{enumerate}
    \item By Lemma~\ref{l:stab}, $\mathbb{K}$ is $(2\kappa, 0.5\gamma)$-strongly stable on $(\hat{A},\hat{B})$, as long as $\varepsilon_{A,B}\leq 0.25\kappa^{-3}\gamma$,
    \item Lemma~\ref{l:ugly} ensures $\|\hat{w_t}\|\leq 2\sqrt{n}\kappa^3\gamma^{-1}W$, as long as $\varepsilon_{A,B}\leq 10^{-3}\kappa^{-10}\gamma^2$.
\end{enumerate}
With the above observations in place, Theorem~\ref{thm:old} guarantees 
\[ J(\mathcal{A}|\hat{A},\hat{B},\{\hat{w}\}) - J(K|\hat{A},\hat{B},\{\hat{w}\}) \leq \poly(m,n,\kappa,\gamma^{-1})GW^2 \sqrt{T}\log T.\]

The last expression in the preceding line can be bound by Lemma~\ref{l:lips}, since Lemma~\ref{l:ugly} also bounds $\|w_t-\hat{w}_t\|\leq 20\sqrt{n}\kappa^{11} \gamma^{-3} W\varepsilon_{A,B}$ whenever $\varepsilon_{A,B}\leq 10^{-3}\kappa^{-10}\gamma^2$ for $t\geq T_0+1$.
\[ |J(K|\hat{A},\hat{B},\{\hat{w}\}) - J(K|A,B,w)| \leq 32 G n\kappa^{11}\gamma^{-3}W^2 + 20^3 T G n\kappa^{22} \gamma^{-7} W^2 \varepsilon_{A,B}\]
Set $\varepsilon_{A,B} = \min(10^{-3}\kappa^{-10}\gamma^2,T^{-\frac{1}{3}})$ in Theorem~\ref{thm:rec} to conclude.
\end{proof}

The regret minimizing algorithm, Phase 2 of Algorithm~\ref{alg:complete}, chooses $M_t$ so as to optimize for the cost of the perturbation-based controller on a {\em fictitious} linear dynamical system $(\hat{A}, \hat{B})$ subject to the perturbation sequence $\{\hat{w}\}$. The following lemma shows that the definition of $\hat{w}_t$ ensures that state-control sequence visited by Algorithm~\ref{alg:complete} coincides with the sequence visited by the regret-minimizing algorithm on the {\em fictitious} system.

\begin{lemma}[Simulation Lemma]\label{l:sim}
Let $\mathcal{A}$ be the algorithm, from \cite{agarwal2019online}, executed in Phase 2, and $(x_t,u_t)$ be the state-control iterates produced by Algorithm~\ref{alg:complete}. Then for $t\geq T_0+1$,
\[x_t=x_t(\mathcal{A}|\hat{A},\hat{B}, \{\hat{w}\}), \quad u_t=u_t(\mathcal{A}|\hat{A},\hat{B}, \{\hat{w}\}),\quad and \quad \sum_{t=T_0+1}^T c_t(x_t,u_t) = J(\mathcal{A}|\hat{A},\hat{B},\{\hat{w}\}).\]

\end{lemma}
\begin{proof}
This proof follows by induction on $x_t=x_t(\mathcal{A}|\hat{A},\hat{B},\{\hat{w}\})$. Note that at the start of $\mathcal{A}$, it is fed the initial state $x_{T_0+1}(\mathcal{A}|\hat{A},\hat{B},\{\hat{w}\}) = x_{T_0+1}$ by the choice of $\hat{w}_{T_0}$. Say that for some $t\geq T_0+1$, it happens that the inductive hypothesis is true. Consequently,
\begin{align*}
    u_{t}(\mathcal{A}|\hat{A},\hat{B},\{\hat{w}\}) &= -\mathbb{K}x_{t}(\mathcal{A}|\hat{A},\hat{B},\{\hat{w}\}) + \sum_{i=1}^H M_t^{[i-1]}\hat{w}_{t-i}= -\mathbb{K}x_{t} + \sum_{i=1}^H M_t^{[i-1]}\hat{w}_{t-i} = u_t.\\
    x_{t+1}(\mathcal{A}|\hat{A},\hat{B},\{\hat{w}\}) &= \hat{A}x_{t}(\mathcal{A}|\hat{A},\hat{B},\{\hat{w}\}) + \hat{B}u_{t}(\mathcal{A}|\hat{A},\hat{B},\{\hat{w}\}) + \hat{w_t} 
    = \hat{A}x_{t} + \hat{B}u_{t} + \hat{w_t} = x_{t+1}
\end{align*} 
This, in turn, implies by choice of $\hat{w}_t$ that the next states produced at the next time steps match.
\end{proof}

The lemma stated below guarantees that the strong stability of $\mathbb{K}$ is approximately preserved under small deviations of the system matrices.

\begin{lemma}[Preservation of Stability]\label{l:stab}
If $\mathbb{K}$ is $(\kappa,\gamma)$-strongly stable for a linear system $(A,B)$, ie. $A-BK=QLQ^{-1}$, then $\mathbb{K}$ is $(\kappa+\varepsilon_{A,B},\gamma -2\kappa^3\varepsilon_{A,B})$-strongly stable for $(\hat{A},\hat{B})$, ie. 
\[ \hat{A}-\hat{B}\mathbb{K} = Q\hat{L}Q^{-1}, \quad \|\hat{A}\|,\|\hat{B}\| \leq \kappa+\varepsilon_{A,B}, \quad \|\hat{L}\|\leq 1-\gamma+2\kappa^3\varepsilon_{A,B}, \]
as long as $\|A-\hat{A}\|,\|B-\hat{B}\|\leq \varepsilon_{A,B}$. Furthermore, in such case, the transforming matrices $Q$ that certify strong stability in both these cases coincide, and it holds $\|\hat{L}-L\|\leq 2\kappa^3\varepsilon_{A,B}$.
\end{lemma}
\begin{proof}
Let $A-B\mathbb{K}= QLQ^{-1}$ with $\|Q\|,\|Q^{-1}\|,\|\mathbb{K}\|,\|A\|,\|B\|\leq \kappa$, $\|L\|\leq 1-\gamma$. Now
\begin{align*}
    \hat{A}-\hat{B}\mathbb{K} &= QLQ^{-1}+(\hat{A}-A)-(\hat{B}-B)\mathbb{K}\\
    &= Q(L+Q^{-1}((\hat{A}-A)-(\hat{B}-B)\mathbb{K})Q)Q^{-1}
\end{align*}
It suffices to note that $\|Q^{-1}((\hat{A}-A)-(\hat{B}-B)\mathbb{K})Q\|\leq 2\kappa^3\varepsilon_{A,B}$.
\end{proof}

The next lemma establishes that if the same linear state-feedback policy is executed on the actual and the {\em fictional} linear dynamical system, the difference between the costs incurred in the two scenarios varies proportionally with some measure of distance between the two systems.

\begin{lemma}[Stability of Value Function]\label{l:lips}
Let $\|A-\hat{A}\|, \|B-\hat{B}\|\leq \varepsilon_{A,B}\leq 0.25 \kappa^{-3}\gamma$, and $K$ be any $(\kappa,\gamma)$-strongly stable controller with respect to $(A,B)$. Then, for any perturbation sequence that satisfies $\|w_t-\hat{w}_t\|\leq \varepsilon_w \leq W_0$ except possibly for the first step where $\|\hat{w}_0\|\leq W_0$\footnote{This handles the case for $\hat{w}_{T_0}$ in Algorithm~\ref{alg:complete}. Assume that $W_0\geq W$.}, it holds
\[ |J(K|\hat{A},\hat{B},\{\hat{w}\}) - J(K|A,B,w)| \leq 10^3TG\kappa^8\gamma^{-3}W_0(\varepsilon_W+W_0\varepsilon_{A,B}) + 32G\kappa^5\gamma^{-2}W_0^2. \]
\end{lemma}
\begin{proof}
Under the action of a linear controller $K$, which is $(\kappa,\gamma)$-strongly stable for $(A,B)$, it holds
\[x_{t+1}(K|A,B,\{w\}) = \sum_{t=0}^T (A-BK)^{i} w_{t-i}.\]
Consequently, $\|x_t(K|A,B,\{w\})\| \leq \kappa^2\gamma^{-1}W$, and, since $K$ is $(2\kappa, 2^{-1}\gamma)$-strongly stable for $(\hat{A}, \hat{B})$ by Lemma~\ref{l:stab},  $\|x_t(K|\hat{A},\hat{B},\{\hat{w}\}) \| \leq 16\kappa^2 \gamma^{-1} W_0$. It follows
\[ |J(\mathcal{A}|\hat{A},\hat{B},\{\hat{w}\}) - J(\mathcal{A}|A,B,w)| \leq 32G\kappa^3\gamma^{-1}W_0 \sum_{t=0}^T\|x_{t+1}(K|A,B,\{w\})-x_{t+1}(K|\hat{A},\hat{B},\{\hat{w}\})\| \]
Finally, by using strong stability of the controller, we have
\begin{align*}
    &\|x_{t+1}(K|A,B,\{w\})-x_{t+1}(K|\hat{A},\hat{B},\{\hat{w}\})\| \leq \sum_{t=0}^T \left\|(A-BK)^{i} w_{t-i}-  (\hat{A}-\hat{B}K)^{i} \hat{w}_{t-i}\right\| \\
    \leq & \sum_{t=0}^T \left( \left\|(A-BK)^{i} w_{t-i} - (A-BK)^{i} \hat{w}_{t-i}\right\| + \left\|(A-BK)^{i} \hat{w}_{t-i}  -(\hat{A}-\hat{B}K)^{i} \hat{w}_{t-i}\right\| \right)\\
    \leq & \kappa^2(1-\gamma)^t W_0 + \kappa^2\gamma^{-1}\varepsilon_w + 16W_0 \kappa^5 \gamma^{-2} \varepsilon_{A,B}
\end{align*}
The last line follows from invocation of Lemma~\ref{l:elad} on $L, L'$ matrices that certify the strongly stability of $K$ on $(A,B)$ and $(\hat{A},\hat{B})$ respectively.
\end{proof}

\begin{lemma}\label{l:elad}
For any matrix pair $L,\Delta L$, such that $\|L\|, \|L+\Delta L\| \leq 1-\gamma$, we have 
$$ \sum_{t=0}^\infty \| (L+\Delta L)^t - L^t \| \leq \gamma^{-2} \| \Delta L \|  $$
\end{lemma}
\begin{proof}
We make the inductive claim that $\|(L+\Delta L)^t-L^t\|\leq t(1-\gamma)^{t-1}\|\Delta L\|$. The truth of this claim for $t=0,1$ is easily verifiable. Assuming the inductive claim for some $t$, observe
\[ \| (L+\Delta L)^{t+1} - L^{t+1} \|  \leq \| L ( (L+\Delta L)^t - L^t) \| +  \| \Delta L (L+\Delta L)^t\|  \leq (t+1) (1-\gamma)^t \|\Delta L\|.\]
Finally, observe that $\sum_{t=0}^\infty t(1-\gamma)^{t-1} \leq \gamma^{-2}$.
\end{proof}

\begin{lemma}\label{l:ugly}
During Algorithm~\ref{alg:complete}.2, it holds, as long as $\varepsilon_{A,B}\leq 10^{-3}\kappa^{-10}\gamma^2$, for any $t\geq T_0+1$
\[ \|x_t\| \leq 4\sqrt{n}\kappa^{10} \gamma^{-3} W, \quad  \|w_t-\hat{w}_t\|\leq 20\sqrt{n}\kappa^{11}\gamma^{-3}W\varepsilon_{A,B}, \quad \text{and} \quad \|\hat{w}_{t-1}\|\leq 2\sqrt{n}\kappa^3\gamma^{-1}W.\]
\end{lemma}
\begin{proof}
For a linear system evolving as as $x_{t+1} = Ax_t + Bu_t + w_t$, if the control is chosen as $u_{t}=-\mathbb{K}x_t + \tilde{u}_t$, the following is true, where $A'=A-B\mathbb{K}$.
\begin{equation}\label{e:2}
x_{t+1} = \sum_{i=0}^{t} (A')^{t-i} (w_{i}+B\tilde{u}_{i}).
\end{equation}
In Phase 2, $\tilde{u}_t$ is chosen as $\tilde{u}_t = \sum_{i=1}^H M_t^{[i-1]}\hat{w}_{t-i}$. We put forward the inductive hypothesis that for all $t\in [T_0+1, t_0]$, we have that $\|x_t\|\leq X := 4\sqrt{n}\kappa^{10}\gamma^{-3}W $, $\|\hat{w}_{t}-w_t\|\leq Y := 20\sqrt{n}\kappa^{11}\gamma^{-3}W\varepsilon_{A,B}$.
If so, $\|\hat{w}_t\|\leq \max(W+Y,\|\hat{w}_{T_0}\|) = W+Y+\sqrt{n}\kappa^3\gamma^{-1}W$ for all $t\leq t_0$, by Lemma~\ref{l:exp}. The following display completes the induction.
\begin{align*}
     \|x_{t_0+1}\| &\leq \kappa^2\gamma^{-1} (W+\kappa^5\gamma^{-1}(W+Y+\sqrt{n}\kappa^3\gamma^{-1}W))\leq X  \\
    \|\hat{w}_{t_0+1}-w_{t_0+1}\| &= \|(A-\hat{A})x_{t_0+1} + (B-\hat{B})u_{t_0+1}\| \\ 
    &\leq \varepsilon_{A,B} ( X +  (\kappa X+\kappa^4\gamma^{-1} (W+Y+\sqrt{n}\kappa^3\gamma^{-1}W) )) \leq Y
\end{align*}
The base case be verified via computation. 
\end{proof}

\section{System Identification via Random Inputs}
This section details the guarantees afforded by the system identification procedure, which attempts to identify the deterministic-equivalent dynamics $(A,B)$ by first identifying matrices of the form $(A')^iB$, where $A'=A-B\mathbb{K}$, and then recovering $A$ by solving a linear system of equations.

\begin{theorem}[System Recovery]\label{thm:rec}
Under assumptions \ref{a:a1}, \ref{a:a3}, \ref{a:a4}, when Algorithm~\ref{alg:B} is run for $T_0$ steps, the output pair $(\hat{A},\hat{B})$ satisfies, with probability $1-\delta$, that $\|\hat{A}-A\|_F, \|\hat{B}-B\|_F\leq \varepsilon_{A,B}$, where
\[
    T_0= 10^3 kmn^2 \kappa^{10}\gamma^{-2} W^2\varepsilon_{A,B}^{-2}\log (kmn\delta^{-1}).
\] 
\end{theorem}

\begin{proof}
Observe that $A'$ is a unique solution of the system of equations (in $X$) presented below.
\[ AC_k = [A'B, (A')^2B, \dots (A')^kB] = X [B, A'B, \dots (A')^{k-1}B] = X C_k\]
Now, if, for all $j$, $\|N_j-(A')^jB\|\leq \varepsilon$, it follows that $\|C_0-C_k\|_F, \|C_1-AC_k\|_F\leq \varepsilon \sqrt{k}$; in addition, $\|\hat{B}-B\|\leq \varepsilon$. By Lemma~\ref{l:axb} on rows, we have
\[\|\hat{A'}-A'\|_F\leq \frac{2\varepsilon\sqrt{km}\kappa}{\sigma_{min}(C_k)-\varepsilon\sqrt{k}}.\]
So, setting $\varepsilon=\frac{\varepsilon_{A,B}}{10\sqrt{km} \kappa^2}$ suffices. This may be accomplished by Lemma~\ref{l:conc}.
\end{proof}

\begin{lemma}\label{l:exp}
When the control inputs are chosen as $u_t=-\mathbb{K}x+\eta_t$, where $\eta_t \sim \{\pm 1\}^n$, it holds
\begin{align*}
    \|x_t\|\leq \sqrt{n}\kappa^3\gamma^{-1}W,\quad 
    \|u_t\|\leq 2\sqrt{n}\kappa^4\gamma^{-1}W,\\
    c_t(x_t,u_t)-\min_{K\in \mathcal{K}} c_t(x^K,u^K)\leq 16G n \kappa^8\gamma^{-2}W^2.
\end{align*}
\end{lemma}
\begin{proof}
In conjuction with Equation~\ref{e:2}, the strong stability of $\mathbb{K}$ suffices to establish this claim.
\begin{align*}
    \|x_t\| \leq (W+\|B\|\sqrt{n}) \|Q\| \|Q^{-1}\|\sum_{i=0}^{t-1}  \|L^i\| \leq (W+\kappa \sqrt{n})\kappa^2 \sum_{i=0}^{t-1} (1-\gamma)^i
\end{align*}
In addition to sub-multiplicativity of the norm, we use that $\sum_{i=0}^{t-1} (1-\gamma)^i \leq \gamma^{-1}$.
\end{proof}

\subsection{Step 1: Moment Recovery}
The following lemma promises an approximate recovery of $(A')^iB$'s through an appeal to arguments involving measures of concentration.

\begin{lemma}\label{l:conc}
Algorithm \ref{alg:B} satisfies for all $j\in [k]$, with probability $1-\delta$ or more
\[ \| N_j - (A')^j B  \| \leq n\kappa^3\gamma^{-1}W\sqrt{\frac{8\log (mnk\delta^{-1})}{T_0-k}}. \]
\end{lemma}
\begin{proof}
Let $N_{j,t}=x_{t+j+1}\eta_t^\top$. With Equation~\ref{e:2}, the fact that $\eta_i$ is zero-mean with isotropic unit covariance, and that it is chosen independently of $\{w_j\},\{\eta_j\}_{j\neq i}$ implies $\mathbb{E}[N_j] = \mathbb{E}[N_{j,t}] = (A')^j B$.

$N_{j,t}$'s are bounded as $\|N_{j,t}\|\leq n\kappa^3 \gamma^{-1}W$. Involving the same instance of $\eta$ in multiple terms, they may not be independent across the second index $t$. To remedy this, define a sequence of centered random variables $\tilde{N}_{j,t}=N_{j,t}-(A')^jB$, and observe that they forms a martingale difference sequence with respect to $\{\eta_t\}$, ie.
\begin{align*}
\mathbb{E}[\tilde{N}_{j,t} | \eta_0,\dots \eta_{t-1}] &= \mathbb{E}[x_{t+j+1} \eta_t^\top | \eta_0,\dots \eta_{t-1}] - (A')^jB = 0.
\end{align*}
Together with a union bound on choice of $j\in [k]$, the application of Matrix Azuma inequality on a standard symmetric dilation of the sequence concludes the proof of the claim.
\end{proof}

\subsection{Step 2: Recovery of System Matrices}
The following is a standard result on the perturbation analysis for linear systems (See \cite{bhatia2013matrix}, for example). A proof is presented here for completeness. 

\begin{lemma}\label{l:axb}
Let $x^*$ be the solution to the linear system $Ax=b$, and $\hat{x}$ be the solution to $(A+\Delta A)x=b+\Delta b$, then as long as $\|\Delta A\|\leq \sigma_{\min}(A)$, it is true that
\[ \|x^*-\hat{x}\| \leq \frac{\|\Delta b\|+\|\Delta A\|\|x^*\|}{\sigma_{min}(A)-\|\Delta A\|}\]
\end{lemma}
\begin{proof}
Observe the (in)equalities that imply the claim.
\begin{align*}
    (A+\Delta A) (\hat{x}-x^*) &= b+\Delta b -Ax^* - \Delta Ax^* = \Delta b-\Delta Ax^* \\
    \|(A+\Delta A)^{-1}\| &\leq \frac{1}{\sigma_{min}(A)-\|\Delta A\|}
\end{align*}
\end{proof}

\section{Conclusions}
We define a generalization of the control setting to the nonstochastic regime in which the dynamics, in addition to the perturbations, are adversarial and unknown to the learner. In this setting we give the first efficient algorithm with sublinear regret, answering an open question posed in \cite{tu2019sample}. 


\section*{Acknowledgements}
We thank Naman Agarwal, Max Simchowitz and Cyril Zhang for helpful discussions. Sham Kakade acknowledges funding from the Washington Research Foundation for Innovation in Data-intensive Discovery and the ONR award N00014-18-1-2247. Karan Singh gratefully acknowledges support from the Porter Ogden Jacobus Fellowship.

\bibliography{main}


\end{document}